\documentclass[10pt,journal,compsoc]{IEEEtran}
\ifCLASSINFOpdf
\else
\fi
\usepackage[brazil, english]{babel}
\usepackage{hyperref}
\usepackage{amsmath,graphicx,amssymb,amsthm,bm,algorithm,algorithmic}
\usepackage{color,url,balance,times,helvet,courier}
\usepackage{booktabs,threeparttable}
\usepackage[square,comma,sort&compress,numbers]{natbib}
\usepackage{threeparttable,multirow,adjustbox,threeparttable}
\newtheorem{lem}{Lemma}
\newtheorem{theorem}{Theorem}
\newtheorem{prop}{Proposition}

\theoremstyle{remark}

\usepackage{changes}
\colorlet{Changes@Color}{red}

\begin{document}
\title{Weighted Sparse Partial Least Squares for Joint Sample and Feature Selection}
\author{Wenwen Min, Taosheng Xu and Chris Ding
\IEEEcompsocitemizethanks{
\IEEEcompsocthanksitem Wenwen Min is with the Yunnan Key Laboratory of Intelligent Systems and Computing, School of Information Science and Engineering, Yunnan University, Kunming 650091, Yunnan, China. E-mail: minwenwen@ynu.edu.cn.

\IEEEcompsocthanksitem Taosheng Xu is with the Institute of Intelligent Machines, Hefei Institutes of Physical Science, Chinese Academy of Sciences, Hefei, 230031, Anhui, China. E-mail: taosheng.x@gmail.com.

\IEEEcompsocthanksitem Chris Ding is with the School of Data Science, The Chinese University of Hong Kong, Shenzhen, China.
E-mail: chrisding@cuhk.edu.cn.
}
\thanks{Manuscript received XX, 2022; revised XX, 2023.}}

\markboth{IEEE TRANSACTIONS ON KNOWLEDGE AND DATA ENGINEERING, 2023}
{Min \MakeLowercase{\textit{et al.}}: Weighted Sparse Partial Least Squares}

\IEEEtitleabstractindextext{
\begin{abstract}
Sparse Partial Least Squares (sPLS) is a common dimensionality reduction technique for data fusion,
which projects data samples from two views by seeking linear combinations with a small number of variables with the maximum variance.
However, sPLS extracts the combinations between two data sets with all data samples so that it cannot detect latent subsets of samples.
To extend the application of sPLS by identifying a specific subset of samples and remove outliers,
we propose an $\ell_\infty/\ell_0$-norm constrained weighted sparse PLS ($\ell_\infty/\ell_0$-wsPLS) method for joint sample and feature selection,
where the $\ell_\infty/\ell_0$-norm constrains are used to select a subset of samples.
We prove that the $\ell_\infty/\ell_0$-norm constrains have the Kurdyka-\L{ojasiewicz}~property so that a globally convergent algorithm is developed to solve it.
Moreover, multi-view data with a same set of samples can be available in various real problems.
To this end, we extend the $\ell_\infty/\ell_0$-wsPLS model and propose two multi-view wsPLS models for multi-view data fusion.
We develop an efficient iterative algorithm for each multi-view wsPLS model and show its convergence property.
As well as numerical and biomedical data experiments demonstrate the efficiency of the proposed methods.
\end{abstract}
}
\maketitle
\IEEEpeerreviewmaketitle

\section{Introduction}
\IEEEPARstart{A} number of paired biomedical omics data from the same batch of patients have been accumulated \cite{hutter2018cancer,welch2019single}.
These invaluable datasets provide us new opportunities to explore cooperative mechanisms between different types of biological molecules.
Partial Least Squares (PLS) is a popular technique for finding correlations across two data matrices with the maximum variance on a new low-dimensional space by learning two latent vectors for both representations \cite{krishnan2011partial,boulesteix2007partial,chen2019solving}.
Recently, PLS has been widely used for data fusion of the paired omics data \cite{rohart2017mixomics},
such as DNA copy number variation and gene expression data, DNA methylation and gene expression data, etc.

However, the classical PLS has the overfitting problem when dealing with high-dimensional but small sample biomedical data.
In addition, the outputs of PLS (i.e., latent vectors) are a combination of all variables, which is difficult to interpret.
To address these problems, sparse PLS (sPLS) methods have been proposed to avoid overfitting problem and improve interpretation ability of model by learning two sparse latent vectors \cite{le2008sparse,chun2010sparse,monteiro2016multiple,li2015discriminative,zhu2020envelope}.
Moreover, some extensions of the sPLS methods have been developed by using the structured sparse penalties, such as structured sparse PLS and matrix factorization methods \cite{liquet2016group,chen2016integrative,min2019group,min2021Novel,min2022structured}.
These extension methods can improve the performance by integrating the structural information of the variables when applying to the biomedical paired omics datasets.
For example, Chen et al. \cite{chen2016integrative} proposed a sparse network-regularized PLS method to identify common modules (co-modules) using the matched gene-expression and drug response data.
However, these sPLS and their variants often perform poorly when applying to high dimensional biomedical datasets with irrelevant features and noisy samples.
The main reason is that they are susceptible to outliers and noisy samples, causing the acquired projection vectors to deviate from the desired direction.
On one hand, most medical datasets are high-dimensional and polluted data due to the influence of experimental conditions and environmental factors.
On the other hand, even the collected patients in the medical datasets from the same disease show heterogeneity in the field of biomedicine, such as cancer disease \cite{dagogo2018tumour}.
So it is very necessary to mine sample-specific co-modules when applying to two view or multi-view data.

To this end, we develop an $\ell_\infty/\ell_0$-norm constrained weighted sparse PLS ($\ell_\infty/\ell_0$-wsPLS) model for joint sample and feature selection, which can simultaneously select features and samples on a paired datasets (Figure \ref{fig-1}).
The use of $\ell_\infty/\ell_0$-norm constraint is used to select a sample subset, while the $\ell_0$-norm constraint is used to select features in our model.
However, it is very difficult to find an effective convergence algorithm in traditional way, to solve the $\ell_\infty/\ell_0$-wsPLS model because the $\ell_\infty/\ell_0$-norm and the $\ell_0$-norm are non-convex and non-smooth.
Fortunately, we prove that the $\ell_\infty/\ell_0$-norm and the $\ell_0$-norm constraints satisfy the Kurdyka-\L{ojasiewicz} (K\L) property.
Inspired by the Proximal Alternating Linearized Minimization (PALM) method \cite{bo2014proximal},
we develop a block proximal gradient algorithm to solve the $\ell_\infty/\ell_0$-wsPLS model and prove that it converges to a critical for any initial point.

Furthermore, multi-view biomedical data have accumulated at an unprecedented rate in recent years \cite{chen2018matrix,meng2016dimension,rappoport2018multi,wang2020scalable}.
These collected multi-view biomedical data presents possibility to study the molecular mechanism of cancer pathogenesis.
Many multi-view learning methods have been proposed to consider the cancer multi-omics data to improve their generalization ability \cite{zhang2019binary,han2022multi,hou2017multi,hu2018adaptive,deng2020multi,zhao2017multi},
but it is still a challenge to mine co-modules between multi-view biomedical data.
To this end, we extend the wsPLS model for multi-view omics data analysis and propose two multi-view weighted sparse PLS (mwsPLS) models with two different schemes.
For each mwsPLS model, we develop an efficient iterative algorithm based on the PALM framework and prove its convergence.

In this paper,
we firstly show that the proposed methods significantly improve the performance of joint sample and feature selection when applying to simulated data.
Secondly, we demonstrate the ability of $\ell_\infty/\ell_0$-wsPLS in identifying sample-specific co-modules when applying to the paired DNA copy number variation and gene expression data, and the paired DNA methylation and gene expression data, respectively.
The results demonstrate that $\ell_\infty/\ell_0$-wsPLS outperforms others methods, which obtains higher correlation coefficients by removing noisy samples.
Thirdly, we also demonstrate the ability of mwsPLS in discovering cancer subtype-specific co-modules when applying the paired miRNA-lncRNA-mRNA expression data of breast cancer.
The results show that mwsPLS can identify some important miRNA-lncRNA-mRNA co-modules which are involved in some breast cancer related biological processes. The main contributions of this paper are highlighted as follows:
\begin{enumerate}
  \item We propose an $\ell_\infty/\ell_0$-norm constrained wsPLS model for joint sample and feature selection, where the $\ell_\infty/\ell_0$-norm constraints are used for sample selection.
  We prove that the $\ell_\infty/\ell_0$-norm constraints satisfy the K\L~property so that a globally convergent algorithm is developed to solve the $\ell_\infty/\ell_0$-wsPLS model.
  \item To integrate more than two datasets, we extend the $\ell_\infty/\ell_0$-wsPLS model and proposed two multi-view wsPLS (mwsPLS) models with two different schemes.
   We develop an efficient iterative algorithm based on the PALM framework and prove its convergence for each mwsPLS model.
  \item The application of $\ell_\infty/\ell_0$-wsPLS and mwsPLS methods and the comparison with the competing methods using the simulated and cancer multi-omic data.
\end{enumerate}

\begin{table}[htp]
\centering
\caption{Summary of notations.}
\begin{adjustbox}{width=1\columnwidth,center}
\begin{tabular}{l|l}
   \hline
   \textbf{Notation}  & \textbf{Meanings}\\
   \hline
   Normal font, e.g., x     & A scalar \\
   Bold lowercase, e.g., $\bm{u}$  & A vector\\
   Bold capital, e.g., $\bm{X}$    & A matrix\\
   \hline
   $\bm{X}_i$  & A $n$-by-$p_i$ matrix\\
   $\bm{X}$    & A $n$-by-$p$ matrix\\
   $\bm{Y}$    & A $n$-by-$q$ matrix\\
   \hline
   $\|\cdot\|$ or $\|\cdot\|_2$ & $\ell_2$-norm for a vector\\
   $\|\cdot\|_1$                & $\ell_1$-norm for a vector\\
   $\|\cdot\|_0$                & $\ell_0$-norm for a vector\\
   $\|\bm{w}\|_\infty$          & $\|\bm{w}\|_\infty =\max_j |w_j|$\\
   $\odot$                      & Element-wise product for two vectors\\
   \hline
\end{tabular}
\end{adjustbox}\label{tab-1}
\end{table}

\begin{figure*}[htbp]
  \centering \includegraphics[width=1\linewidth]{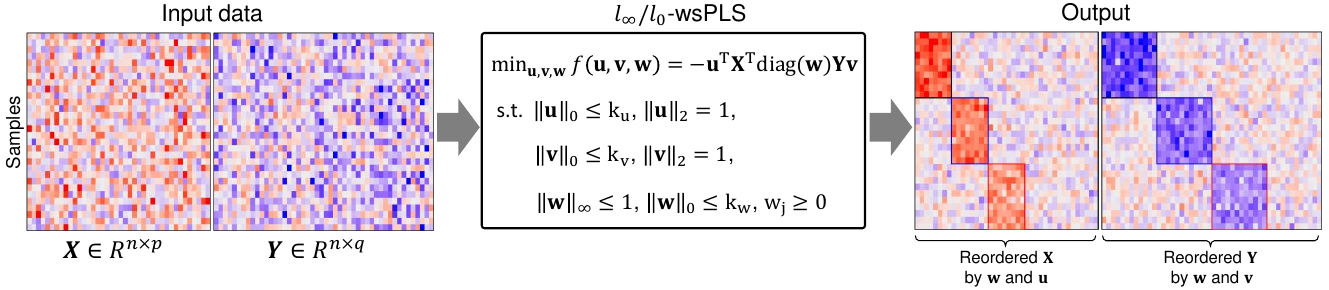}
  \caption{Overview of the proposed  $\ell_\infty/\ell_0$-norm constrained wsPLS model for joint sample and feature selection, where the $\ell_\infty/\ell_0$-norm constrains are used to select a subset of samples.
  The $\ell_\infty/\ell_0$-wsPLS model can not only obtain two sparse canonical vectors $\bm{u}$ and $\bm{v}$, but also identify a set of samples based on those non-zero elements of $\bm{w}$. In this toy example, $\ell_\infty/\ell_0$-wsPLS is used to extract three co-modules.
  }\label{fig-1}
\end{figure*}

\section{Weighted Sparse PLS (wsPLS)}
\subsection{wsPLS Model}
Assume that there are two data matrices $\bm{X}\in \mathbb{R}^{n\times p}$ and $\bm{Y} \in \mathbb{R}^{n\times q}$ with a same set of samples and their columns are standardized to have mean zero and variance unit. PLS is equivalent to solving the following optimization problem \cite{rosipal2005overview}:
\begin{equation}\label{equ-1}
\begin{aligned}
& \underset{\bm{u},\bm{v}}{\text{maximize}} && \bm{u}^T\bm{X}^T\bm{Y}\bm{v} \\
& \text{subject to}  && \|\bm{u}\| = \|\bm{v}\| = 1.
\end{aligned}
\end{equation}
The basic idea of PLS is to obtain two projected vectors $\bm{u}$ and $\bm{v}$ by maximizing the covariance between the latent vectors $\bm{Xu}$ and $\bm{Yv}$.
Specifically, a sparse PLS (sPLS) with $\ell_0$-norm constraint is formulated as:
\begin{equation}\label{equ-2}
\begin{aligned}
& \underset{\bm{u},\bm{v}}{\text{maximize}} && \bm{u}^T\bm{X}^T\bm{Y}\bm{v} \\
& \text{subject to}  && \|\bm{u}\|_0\leq k_u, \|\bm{v}\|_0\leq k_v, \|\bm{u}\| = \|\bm{v}\| = 1,
\end{aligned}
\end{equation}
where $k_u$ and $k_v$ are two parameters to control the sparsity of $\bm{u}$ and $\bm{v}$. The above model is also called sparse diagonal CCA problem \cite{asteris2016simple,wenwen2018sparse}.

We observe that the objective function in Eq. (\ref{equ-2}) can be written as $\sum_{i=1}^n (\bm{Xu})_i(\bm{Yv})_i$.
To consider the difference of samples, we set a weighted coefficient for each sample and modify the above objective function as $\sum_{i=1}^n w_i (\bm{Xu})_i(\bm{Yv})_i=\bm{u}^T\bm{X}^T \mbox{diag}(\bm{w}) \bm{Y}\bm{v}$ and propose a weighted sparse PLS model with the $\ell_\infty/\ell_0$-norm constraints, which is denoted as ($\ell_\infty/\ell_0$-wsPLS or wsPLS for short):
\begin{equation}\label{equ-3}
\begin{aligned}
& \underset{\bm{u},\bm{v},\bm{w}}{\text{maximize}} && \bm{u}^T\bm{X}^T \mbox{diag}(\bm{w}) \bm{Y}\bm{v} \\
& \text{subject to}  && \|\bm{u}\|_0\leq k_u, \|\bm{v}\|_0\leq k_v, \|\bm{u}\| = \|\bm{v}\| = 1,\\
&  && \|\bm{w}\|_\infty \leq 1, \|\bm{w}\|_0\leq k_w, w_j\geq 0~\forall j
\end{aligned}
\end{equation}
where $\|\bm{w}\|_\infty = \max_j |w_j| $ and $\mbox{diag}(\bm{w}) \in \mathbb{R}_+^{n\times n}$ is a diagonal matrix with the diagonal entries $w_1, w_2, \cdots, w_n$ indicating the weighted coefficients of samples and the $\ell_\infty/\ell_0$-norm constraints are used to select a subset of samples.
The mathematical notations are summarized in Table \ref{tab-1}.

\subsection{Optimization Algorithm}
Eq. (\ref{equ-3}) is equivalent to the following optimization problem because $\bm{u}^T\bm{X}^T\mbox{diag}(\bm{w})\bm{Y}\bm{v} = \bm{w}^T[(\bm{X}\bm{u})\odot(\bm{Y}\bm{v})]$:
\begin{equation}\label{equ-4}
\begin{aligned}
& \underset{\bm{u},\bm{v},\bm{w}}{\text{minimize}} && f(\bm{u},\bm{v},\bm{w}) = -\bm{w}^T[(\bm{X}\bm{u})\odot(\bm{Y}\bm{v})] \\
& \text{subject to}  && \|\bm{u}\|_0\leq k_u, \|\bm{v}\|_0\leq k_v, \|\bm{u}\| = \|\bm{v}\| = 1,\\
&  && \|\bm{w}\|_\infty \leq 1, \|\bm{w}\|_0\leq k_w, w_j\geq 0~\forall j
\end{aligned}
\end{equation}
Obviously, the above $\ell_\infty/\ell_0$-wsPLS model reduces to sPLS model when all elements of $\bm{w}$ are one.
We observe that $f(\bm{u},\bm{v},\bm{w})$ is also equal to $-\bm{u}^T \bm{X}^T[\bm{w}\odot(\bm{Y}\bm{v})]$ or $-\bm{v}^T \bm{Y}^T[\bm{w}\odot(\bm{X}\bm{u})]$. Therefore, we have
\begin{equation}\label{equ-5}
\begin{aligned}
&\nabla_{\bm{u}} f(\bm{u},\bm{v},\bm{w})&=& -\bm{X}^T[\bm{w}\odot(\bm{Y}\bm{v})],\\
&\nabla_{\bm{v}} f(\bm{u},\bm{v},\bm{w})&=& -\bm{Y}^T[\bm{w}\odot(\bm{X}\bm{u})],\\
&\nabla_{\bm{w}} f(\bm{u},\bm{v},\bm{w})&=& -(\bm{X}\bm{u})\odot(\bm{Y}\bm{v}).
\end{aligned}
\end{equation}
Furthermore, we know that the Hessian matrices of $f(\bm{u},\bm{v},\bm{w})$ with respect to $\bm{u}$, $\bm{v}$ and $\bm{w}$ are zero matrix, respectively. Thus, $\nabla_{\bm{u}} f(\bm{u},\bm{v},\bm{w})$, $\nabla_{\bm{v}} f(\bm{u},\bm{v},\bm{w})$ and $\nabla_{\bm{w}} f(\bm{u},\bm{v},\bm{w})$ are Lipschitz continuous, and their Lipschitz constants can be any constant greater than zero ($L_u>0$, $L_v>0$, and $L_w>0$).

Recently, the PALM algorithm \cite{bo2014proximal} has been proposed to solve a class of non-convex and non-smooth problems.
Based on the PALM framework, we develop a block proximal gradient algorithm to solve (\ref{equ-4}) by iteratively updating each variable when fixing others.

\textbf{1) Computation of $\bm{u}^{t+1}$ using $\bm{u}^{t}$, $\bm{v}^{t}$ and $\bm{w}^{t}$.}
We perform the proximal gradient step for $\bm{u}^{t+1}$ when $\bm{v}$ and $\bm{w}$ are fixed:
\begin{equation}\label{equ-6}
\begin{aligned}
& \underset{\bm{u}}{\text{minimize}} && \|\bm{u} - \overline{\bm{u}}\|_2^2  \\
& \text{subject to}  && \|\bm{u}\|_0\leq k_u, \|\bm{u}\| = 1,\\
\end{aligned}
\end{equation}
where $\overline{\bm{u}} = \bm{u}^{t} - \frac{1}{L_u} \nabla_{\bm{u}} f(\bm{u}^{t}, \bm{v}^{t}, \bm{w}^{t})$.
To solve the above problem (\ref{equ-6}), we define an $\ell_0$-ball projection problem as Eq. (\ref{equ-7}) and give its closed-form solution in Proposition \ref{prop-1}.
\begin{equation}\label{equ-7}
 \bm{x}^* = \underset{\bm{x}}{\text{argmin}} \{\|\bm{x} - \bm{z}\|_2^2: \|\bm{x}\|_0 < k, \|\bm{x}\| = 1\}
\end{equation}
\begin{prop}\label{prop-1}
Suppose that $\bm{z}$ is a no-zero vector, then the solution of problem (\ref{equ-7}) is $\bm{x}^* =  \frac{T_k(\bm{z})}{\|T_k(\bm{z})\|}$.
$T_k(\bm{z})$ keeps the top $k$ entries with largest absolute values of $\bm{z}$ and the other entries are zeros.
\end{prop}
Based on Proposition \ref{prop-1}, we obtain the closed-form solution of Eq. (\ref{equ-6}) as follows:
\begin{equation}\label{equ-8}
   \bm{u}^{t+1} := \frac{T_{k_u}(\overline{\bm{u}})}{\|T_{k_u}(\overline{\bm{u}})\|},
\end{equation}
where $\overline{\bm{u}} = \bm{u}^{t} - \frac{1}{L_u} \nabla_{\bm{u}} f(\bm{u}^{t}, \bm{v}^{t}, \bm{w}^{t})$, the partial derivatives $\nabla_{\bm{u}} f(\bm{u},\bm{v},\bm{w}) = -\bm{X}^T[\bm{w}\odot(\bm{Y}\bm{v})]$ and the Lipschitz constant $L_u$ can be set to any constant greater than zero.

\textbf{2) Computation of $\bm{v}^{t+1}$ using $\bm{u}^{t+1}$, $\bm{v}^{t}$ and $\bm{w}^{t}$.}
We perform the proximal gradient step for $\bm{v}^{t+1}$ when $\bm{u}$ and $\bm{w}$ are fixed:
\begin{equation}\label{equ-9}
\begin{aligned}
& \underset{\bm{v}}{\text{minimize}} && \|\bm{v} - \overline{\bm{v}}\|_2^2  \\
& \text{subject to}  && \|\bm{v}\|_0\leq k_v, \|\bm{v}\| = 1,\\
\end{aligned}
\end{equation}
where $\overline{\bm{v}} = \bm{v}^{t} - \frac{1}{L_v} \nabla_{\bm{u}} f(\bm{u}^{t+1}, \bm{v}^{t}, \bm{w}^{t})$.
Based on the Proposition \ref{prop-1}, we obtain the closed-form solution of Eq. (\ref{equ-9}) as follows:
\begin{equation}\label{equ-10}
   \bm{v}^{t+1} := \frac{T_{k_v}(\overline{\bm{v}})}{\|T_{k_v}(\overline{\bm{v}})\|},
\end{equation}
where $\overline{\bm{v}} = \bm{v}^{t} - \frac{1}{L_v} \nabla_{\bm{u}} f(\bm{u}^{t+1}, \bm{v}^{t}, \bm{w}^{t})$, the partial derivatives $\nabla_{\bm{v}} f(\bm{u},\bm{v},\bm{w}) = -\bm{Y}^T[\bm{w}\odot(\bm{X}\bm{u})]$ and the Lipschitz constant $L_v$ can be set to any constant greater than zero.

\textbf{3) Computation of $\bm{w}^{t+1}$ using $\bm{u}^{t+1}$, $\bm{v}^{t+1}$ and $\bm{w}^{t}$.}
We perform the proximal gradient step for $\bm{w}^{t+1}$ when $\bm{u}$ and $\bm{v}$ are fixed:
\begin{equation}\label{equ-11}
\begin{aligned}
& \underset{\bm{w}}{\text{minimize}} && \|\bm{w} - \overline{\bm{w}}\|_2^2  \\
&  && \|\bm{w}\|_0\leq k_w, \|\bm{w}\|_\infty \leq 1, w_j\geq 0~\forall j
\end{aligned}
\end{equation}
where $\overline{\bm{w}} = \bm{w}^{t} - \frac{1}{L_w} \nabla_{\bm{w}} f(\bm{u}^{t+1}, \bm{v}^{t+1}, \bm{w}^{t})$.
To solve the above model (\ref{equ-11}), we define a projection function $\widehat{\bm{w}}:= P_{\infty,+}(\bm{w})$ for an arbitrary vector $\bm{w}$, which satisfies as follows for all $j$:
\begin{equation}\label{equ-12}
\widehat{w}_j = \left\{
\begin{array}{cl}
  0,   &\text{if}~w_j \leq 0, \\
  1,   &\text{if}~w_j \geq 1, \\
  w_i, &\text{otherwise}.
\end{array}
\right.
\end{equation}
We obtain the closed-form solution of Eq. (\ref{equ-11}) using the following Proposition \ref{prop-2}.
\begin{prop}\label{prop-2}
The optimal solution of problem (\ref{equ-11}) is $\bm{w}^* = T_{k_w}(P_{\infty,+}(\overline{\bm{w}}))$.
\end{prop}

Based on Proposition \ref{prop-2}, we obtain the update rule of $\bm{w}^{t+1}$:
\begin{equation}\label{equ-13}
   \bm{w}^{t+1} := T_{k_w}(P_{\infty,+}(\overline{\bm{w}})),
\end{equation}
where $\overline{\bm{w}} = \bm{w}^{t} - \frac{1}{L_w} \nabla_{\bm{w}} f(\bm{u}^{t+1}, \bm{v}^{t+1}, \bm{w}^{t})$, $P_{\infty,+}(\cdot)$ is defined in Eq. (\ref{equ-12}), $\nabla_{\bm{w}} f(\bm{u},\bm{v},\bm{w}) = -(\bm{X}\bm{u})\odot(\bm{Y}\bm{v})$ and the Lipschitz constant $L_w$ can be set to any constant greater than zero.
\begin{algorithm}[h]
\caption{$\ell_\infty/\ell_0$-norm constrained wsPLS ($\ell_\infty/\ell_0$-wsPLS)} \label{Alg-1}
\begin{algorithmic}[1]
\REQUIRE Data matrices \{$\bm{X}\in \mathbb{R}^{n\times p}$, $\bm{Y}\in \mathbb{R}^{n\times q}$\}, parameters \{$k_u$, $k_v$, $k_w$\}, and Lipschitz constants \{$L_u$, $L_v$, $L_w$\}.
\ENSURE $\bm{u}$, $\bm{v}$, and $\bm{w}$.
\STATE Each column of $\bm{X}$ and $\bm{Y}$ is normalized to have mean zero and variance one.
\STATE Initialize $(\bm{u}^{0}, \bm{v}^{0}, \bm{w}^{0})$ and set $t=0$.
\REPEAT
\STATE Compute $\bm{u}^{t+1}$ according to Eq. (\ref{equ-8}).
\STATE Compute $\bm{v}^{t+1}$ according to Eq. (\ref{equ-10}).
\STATE Compute $\bm{w}^{t+1}$ according to Eq. (\ref{equ-13}).
\STATE $t = t + 1$
\UNTIL convergence.
\RETURN $\bm{u}:=\bm{u}^{t}$, $\bm{v}:=\bm{v}^{t}$ and $\bm{w}:=\bm{w}^{t}$.
\end{algorithmic}
\end{algorithm}

Combing (\ref{equ-8}), (\ref{equ-10}) and (\ref{equ-13}), we propose the block proximal gradient algorithm to solve Eq. (\ref{equ-4}) (See Algorithm \ref{Alg-1}).

\textbf{Initialization}.
We assume that all the samples have equal contribution and initialize $\bm{w}$ as $\bm{w}^{0} = \bm{1}$ (with all entries are ones) in Algorithm \ref{Alg-1}.
Furthermore, Algorithm \ref{Alg-1} is run with multiple random initial points $\bm{u}^{0}$ and $\bm{v}^{0}$.
Finally, the model with the optimal objective function value is selected.

\textbf{Stopping Condition.}
We can use the following ways to stop the algorithm: (i) the stopping criterion can be set as follows:
\begin{equation}\label{equ-14}
\|\bm{u}^{t}-\bm{u}^{t-1}\| + \|\bm{v}^{t}-\bm{v}^{t-1}\| + \|\bm{w}^{t}-\bm{w}^{t-1}\| < 10^{-5},
\end{equation}
where $\bm{u}^{t}$, $\bm{v}^{t}$, and $\bm{w}^{t}$ represent the $t$-th iteration of $\bm{u}$, $\bm{v}$, and $\bm{w}$ respectively.
(ii) the stopping criteria can be set as follows:
\begin{equation}\label{equ-15}
  \frac{|\mbox{obj}(t)-\mbox{obj}(t-1)|}{\mbox{obj}(t-1)} < 10^{-5},
\end{equation}
where $\mbox{obj}(t)$ is the objective function value after the $t$-th iteration.

\subsection{Complexity Analysis}
The computational complexity of Algorithm \ref{Alg-1} mainly depends on the computation of Eq. (\ref{equ-8}), Eq. (\ref{equ-10}) and Eq. (\ref{equ-13}).
Computing Eq. (\ref{equ-8})  is mainly composed of  the computation of $\nabla_{\bm{u}} f(\bm{u},\bm{v},\bm{w}) = -\bm{X}^T[\bm{w}\odot(\bm{Y}\bm{v})]$ which consumes $\mathcal{O}(np + nq + n)$.
Computing Eq. (\ref{equ-10}) is mainly composed of  the computation of $\nabla_{\bm{v}} f(\bm{u},\bm{v},\bm{w}) = -\bm{Y}^T[\bm{w}\odot(\bm{X}\bm{u})]$ which consumes $\mathcal{O}(nq + np + n)$.
Similarly, Computing Eq. (\ref{equ-13}) is mainly composed of the computation of $\nabla_{\bm{w}} f(\bm{u},\bm{v},\bm{w}) = -(\bm{X}\bm{u})\odot(\bm{Y}\bm{v})$ which consumes $\mathcal{O}(nq + np + n)$.
Overall, the complexity of Algorithm \ref{Alg-1} is $\mathcal{O}(tnp + tnq + tn)$, where $t$ is the number of iterations and it is empirically set to 20 in our experiments.

\subsection{Convergence Analysis}
In this subsection, we show that Algorithm \ref{Alg-1} has theoretical convergence guarantees (Theorem \ref{theorem-1}).

\begin{lem} \label{lem1}
$\Phi = -\bm{u}^T\bm{X}^T \mbox{diag}(\bm{w}) \bm{Y}\bm{v} + \delta_{\|\bm{u}\| = 1} + \delta_{\|\bm{u}\|_0 \leq k_u} + \delta_{\|\bm{v}\| = 1} + \delta_{\|\bm{v}\|_0 \leq k_v}+
\delta_{\|\bm{w}\|_\infty \leq 1} + \delta_{\|\bm{w}\|_0 \leq k_w} + \delta_{\bm{w} \geq 0}$ has the Kurdyka-\L{ojasiewicz} (K\L) property.
\end{lem}
\begin{proof}
Inspired by \cite{bo2014proximal}, we have
(1) $\delta_{\|\bm{u}\|_0 \leq k_u}$ and $\delta_{\|\bm{u}\|^2 = 1}$ are semi-algebraic functions;
(2) $\|\bm{w}\|_\infty$ ball is convex, so $\delta_{\|\bm{w}\|_\infty \leq 1}$ is a semi-algebraic function;
(3) $\bm{w} \geq 0$ is a convex constraint, so $\delta_{\bm{w} \geq 0}$ is a semi-algebraic function. Then,
based on Theorem 3 in \cite{bo2014proximal}, $\Phi$ is a semi-algebraic function and has the K\L~property.
\end{proof}
\begin{theorem}\label{theorem-1}
(Globally convergence of Algorithm \ref{Alg-1})
Let $\{(\bm{u}^{t}, \bm{v}^{t}, \bm{w}^{t})\}_{t \in \mathbb{N}}$ be a sequence which are generated for any initial point $(\bm{u}^{0}, \bm{v}^{0}, \bm{w}^{0})$ by using Algorithm \ref{Alg-1}. Then the objective function value $f(\bm{u}^{t}, \bm{v}^{t}, \bm{w}^{t})$ is non-increasing and $\{(\bm{u}^{t}, \bm{v}^{t}, \bm{w}^{t})\}_{t \in \mathbb{N}}$ converges to a critical point.
\end{theorem}
\begin{proof}
References \cite{bo2014proximal} has shown that the PALM algorithm monotonically increasing converges to a critical point if the objective function satisfied the K\L~property. Because Lemma \ref{lem1} has shown that the objective function of the $\ell_\infty/\ell_0$-wsPLS model satisfies the K\L~property and Algorithm \ref{Alg-1} is an algorithm based on the PALM framework, we prove that Algorithm \ref{Alg-1}, i.e., Theorem \ref{theorem-1} holds.
\end{proof}

\subsection{Determination of co-modules}
A co-module is decide with two steps: (1) extract a submatrix in the data marix $\bm{X} \in \mathbb{R}^{n \times p}$, where rows and columns are corresponding to non-zero entries in $\bm{w}$ and $\bm{u}$, respectively; (2) extract another submatrix in the data matrix $\bm{Y} \in \mathbb{Y}^{n \times q}$, where rows and columns are corresponding to non-zero entries in $\bm{w}$ and $\bm{v}$, respectively. The combination of two extracted sub-matrices is called a co-module (see Figure \ref{fig-1}).

We explain how wsPLS identify multiple co-modules. After identifying the first co-module using Algorithm \ref{Alg-1}, we update matrices $\bm{X}$ and $\bm{Y}$ by removing their rows (samples) in the first co-module, and we repeatedly apply  Algorithm \ref{Alg-1} to identify the next co-module on the updated data matrices (see the three co-modules in the rightmost subgraph of Figure \ref{fig-1}).

\section{Multi-view Weighted Sparse PLS}
To identify co-modules on the multi-view data with more than three data matrices, we extend wsPLS and propose two multi-view wsPLS (mwsPLS) models with different schemes.

\subsection{Scheme 1 for mwsPLS using a sum way}
Assume that multiple matrices $\bm{X}_i\in \mathbb{R}^{n\times p_i}$ ($i=1,\cdots,m$) across a common sample set are given and their columns are normalized with zero-mean and unit-variance.
Similar to \cite{witten2009extensions},
we propose the multi-view wsPLS (mwsPLS) model using a sum way, which is called mwsPLS-scheme1 for short as follows:
\begin{equation}\label{equ-16}
\begin{aligned}
& \underset{\bm{u}_1,\cdots,\bm{u}_m, \bm{w}}{\text{maximize}} && \sum_{i < j}\bm{u}_i^T\bm{X}_i^T\mbox{diag}(\bm{w})\bm{X}_j\bm{u}_j \\
& \text{subject to} &&\|\bm{u}_i\| = 1,\|\bm{u}_i\|_0\leq k_i~\forall  i\\
&                   &&\|\bm{w}\|_\infty \leq 1, \|\bm{w}\|_0\leq k_w, \bm{w}\geq 0.\\
\end{aligned}
\end{equation}
Eq. (\ref{equ-16}) is equivalent to
\begin{equation}\label{equ-17}
\begin{aligned}
& \underset{\bm{u}_1,\cdots,\bm{u}_m, \bm{w}}{\text{minimize}} && f(\bm{u},\bm{w}) = - \sum_{i < j}\bm{w}^T[(\bm{X}_i\bm{u}_i)\odot(\bm{X}_j\bm{u}_j)] \\
& \text{subject to} &&\|\bm{u}_i\| = 1,\|\bm{u}_i\|_0\leq k_i~\forall  i\\
&                   &&\|\bm{w}\|_\infty \leq 1, \|\bm{w}\|_0\leq k_w, \bm{w}\geq 0.\\
\end{aligned}
\end{equation}

\textbf{Optimization.}
We observe that the objective function of Eq. (\ref{equ-17}) satisfies
$f(\bm{u},\bm{w}) = - \sum\limits_{i < j}\bm{w}^T[(\bm{X}_i\bm{u}_i)\odot(\bm{X}_j\bm{u}_j)] = -\bm{u}_i^T\bm{X}_i^T[\sum\limits_{i\neq j} \bm{w}\odot(\bm{X}_j\bm{u}_j)]+C$,
where $C$ is independent for any $\bm{u}_i$.
So, we can calculate the partial derivatives of $\bm{u}_i$ ($i=1,\cdots,m$) and $\bm{w}$ as follows:
\begin{equation}\label{equ-18}
\begin{aligned}
&\nabla_{\bm{u}_i} f(\bm{u},\bm{w})&=& - \bm{X}_i^T\sum\limits_{j \neq i} [\bm{w}\odot(\bm{X}_j\bm{u}_j)],\\
&\nabla_{\bm{w}} f(\bm{u},\bm{w})  &=& - \sum_{i < j}[(\bm{X}_i\bm{u}_i)\odot(\bm{X}_j\bm{u}_j)].
\end{aligned}
\end{equation}
Similar to Algorithm \ref{Alg-1}, we develop the block proximal gradient algorithm to solve (\ref{equ-17}) by using iteratively updating way.
The detail algorithm procedure is shown in Algorithm \ref{Alg-2} where Eq. (\ref{equ-19}) and Eq. (\ref{equ-20}) are solved using Proposition \ref{prop-1} and Proposition \ref{prop-2}, respectively.

\begin{algorithm}[htp]
	\caption{Multi-view wsPLS (mwsPLS-scheme1) in Eq. (\ref{equ-16})} \label{Alg-2}
	\begin{algorithmic}[1]
		\REQUIRE Data matrices $\bm{X}_i \in \mathbb{R}^{n\times p_i}$ ($i=1, 2,\cdots, m$), parameters \{$k_1,\cdots,k_m$ and $k_w$\}, and Lipschitz constants \{$L_{u_1},\cdots,L_{u_m}$ and $L_w$\}.
		\ENSURE $\bm{u}_1,\cdots,\bm{u}_m$ and $\bm{w}$.
		\STATE Standardize the columns of $\bm{X}_i$ ($i=1,\cdots,m$).
		\STATE Initialize $(\bm{u}_i^{0},\cdots,\bm{u}_m^{0}, \bm{w}^{0})$ and set $t=0$.
		\REPEAT
		\STATE Compute $\bm{u}_i^{t+1}$ (for $i=1,2,\cdots,m$) by solving the following optimization problem:
		\begin{equation}\label{equ-19}
			\begin{aligned}
				& \underset{\bm{u}}{\text{minimize}} && \|\bm{u} - \overline{\bm{u}}\|_2^2  \\
				& \text{subject to}  && \|\bm{u}\|_0\leq k_i, \|\bm{u}\| = 1,\\
			\end{aligned}
		\end{equation}
		where $\overline{\bm{u}} = \bm{u}_i^{t} - \frac{1}{L_{u_i}} \nabla_{\bm{u}_i} f(\bm{u}^{t},\bm{w}^{t})$ and the partial derivatives $\nabla_{\bm{u}_i} f(\bm{u},\bm{w}) = - \bm{X}_i^T\sum_{j \neq i} [\bm{w}\odot(\bm{X}_j\bm{u}_j)]$.
		\STATE Compute $\bm{w}^{t+1}$ by solving the following problem:
		\begin{equation}\label{equ-20}
			\begin{aligned}
				& \underset{\bm{w}}{\text{minimize}} && \|\bm{w} - \overline{\bm{w}}\|_2^2  \\
				&  && \|\bm{w}\|_0\leq k_w, \|\bm{w}\|_\infty \leq 1, w_j\geq 0~\forall j
			\end{aligned}
		\end{equation}
		where $\overline{\bm{w}} = \bm{w}^{t} - \frac{1}{L_{w}} \nabla_{\bm{w}} f(\bm{u}^{t+1},\bm{w}^{t})$ and the partial derivatives $\nabla_{\bm{w}} f(\bm{u},\bm{w}) = - \sum_{i < j}[(\bm{X}_i\bm{u}_i)\odot(\bm{X}_j\bm{u}_j)]$.
		\STATE $t = t + 1$
		\UNTIL convergence (e.g., $\frac{|\mbox{obj}(t)-\mbox{obj}(t-1)|}{\mbox{obj}(t-1)} < 10^{-5}$).
		\RETURN $\bm{u}_i:=\bm{u}_i^{t}$ for all $i$ and $\bm{w}:=\bm{w}^{t}$.
	\end{algorithmic}
\end{algorithm}

\textbf{Initialization}.
We set $\bm{w}^{0} = \bm{1}$ (all entries are ones) in Algorithm \ref{Alg-2} and randomly initialize \{$\bm{u}_i^{0},~i=1,\cdots,m$\} from a normal distribution.
Algorithm \ref{Alg-2} is run with multiple random initial points and the model with the optimal value is selected.

\textbf{Complexity Analysis.}
The computational burden of Algorithm \ref{Alg-2} mainly depends on the computation of Eq. (\ref{equ-19}) and Eq. (\ref{equ-20}).
The computation of Eq. (\ref{equ-19}) is mainly composed of the computation of $\nabla_{\bm{u}_i} f(\bm{u},\bm{w}) = - \bm{X}_i^T\sum_{j \neq i} [\bm{w}\odot(\bm{X}_j\bm{u}_j)]$ which consumes $\mathcal{O}(\sum_i np_i)$. The computation of Eq. (\ref{equ-20}) is mainly composed of the computation of $\nabla_{\bm{w}} f(\bm{u},\bm{w}) = - \sum_{i < j}[(\bm{X}_i\bm{u}_i)\odot(\bm{X}_j\bm{u}_j)]$ which consumes $\mathcal{O}((m-1)\sum_i np_i)$ because $\sum_{i < j}(np_i+np_j) = (m-1)\sum_i np_i$.
In total, the complexity of Algorithm \ref{Alg-2} is $\mathcal{O}(tmnp_1+tmnp_2+\cdots+tmnp_m)$, where $t$ is the number of iterations and it is empirically set to 20 in our experiments.

\textbf{Convergence Analysis.}
Similar to Lemma \ref{lem1}, we can prove that the objective function of the mwsPLS-scheme1 model satisfies the K\L~property. Therefore, Algorithm \ref{Alg-2} also has theoretical convergence.

\subsection{Scheme 2 for mwsPLS using a product way}
We observe $\bm{u}_1^T\bm{X}_1^T\mbox{diag}(\bm{w})\bm{X}_2\bm{u}_2 = \bm{w}^T[(\bm{X}_1\bm{u}_1)\odot(\bm{X}_2\bm{u}_2)]$,
so that we intuitively extend wsPLS for multi-view data analysis using a product way.
Herein, we propose the second mwsPLS model using the product way, which is called mwsPLS-scheme2 for short as follows:

\begin{equation}\label{equ-21}
\begin{aligned}
& \underset{\bm{u}_1,\cdots,\bm{u}_m, \bm{w}}{\text{minimize}} && f(\bm{u},\bm{w}) = - \bm{w}^T\big[\bigodot\limits_{i=1}^m(\bm{X}_i\bm{u}_i)\big] \\
& \text{subject to} &&\|\bm{u}_i\| = 1,\|\bm{u}_i\|_0\leq k_i~\forall  i\\
&                   &&\|\bm{w}\|_\infty \leq 1, \|\bm{w}\|_0\leq k_w, \bm{w}\geq 0,\\
\end{aligned}
\end{equation}
where $\bigodot\limits_{i=1}^m(\bm{X}_i\bm{u}_i) = (\bm{X}_1\bm{u}_1)\odot(\bm{X}_2\bm{u}_2)\odot \cdots \odot(\bm{X}_m\bm{u}_m)$.

\textbf{Optimization.}
Formally, the objective function in Eq. (\ref{equ-21}) is equivalent to
\begin{equation}\label{equ-22}
f(\bm{u},\bm{w}) = -\bm{u}_i^T [\bm{X}_i^T(\bm{w}\odot\bm{z}_{i})],
\end{equation}
where $\bm{z}_{i}=\bigodot\limits_{j\neq i}^m(\bm{X}_j\bm{u}_j)$.
Therefore, we calculate the partial derivatives for $\bm{u}_i$ ($i=1,\cdots,m$) and $\bm{w}$ as follows:
\begin{equation}\label{equ-23}
	\begin{aligned}
		&\nabla_{\bm{u}_i} f(\bm{u},\bm{w})&=& -\bm{X}_i^T(\bm{w}\odot\bm{z}_{i}), \forall i\\
		&\nabla_{\bm{w}} f(\bm{u},\bm{w})&=& -\bigodot\limits_{i=1}^m(\bm{X}_i\bm{u}_i).
	\end{aligned}
\end{equation}
Finally, we solve the optimization problem (\ref{equ-21}) using a similar manner as Algorithm \ref{Alg-2} and the detail procedure is shown in Algorithm \ref{Alg-3}.

\begin{algorithm}[htp]
\caption{Multi-view wsPLS (mwsPLS-scheme2) in Eq. (\ref{equ-21})} \label{Alg-3}
\begin{algorithmic}[1]
\REQUIRE Data matrices $\bm{X}_i \in \mathbb{R}^{n\times p_i}$ ($i=1, 2,\cdots, m$), parameters \{$k_1,\cdots,k_m$ and $k_w$\} and Lipschitz constants \{$L_{u_1},\cdots,L_{u_m}$ and $L_w$\}.
\ENSURE $\bm{u}_1,\cdots,\bm{u}_m$ and $\bm{w}$.
\STATE Standardize the columns of $\bm{X}_i$ ($i=1,\cdots,m$).
\STATE Initialize $(\bm{u}_i^{0},\cdots,\bm{u}_m^{0}, \bm{w}^{0})$ and set $t=0$.
\REPEAT
\STATE Compute $\bm{u}_i^{t+1}$ (for $i=1,2,\cdots,m$) by solving the following optimization problem:
\begin{equation}\label{equ-24}
	\begin{aligned}
		& \underset{\bm{u}}{\text{minimize}} && \|\bm{u} - \overline{\bm{u}}\|_2^2  \\
		& \text{subject to}  && \|\bm{u}\|_0\leq k_i, \|\bm{u}\| = 1,\\
	\end{aligned}
\end{equation}
where $\overline{\bm{u}} = \bm{u}_i^{t} - \frac{1}{L_{u_i}} \nabla_{\bm{u}_i} f(\bm{u}^{t},\bm{w}^{t})$, the partial derivatives $\nabla_{\bm{u}_i} f(\bm{u},\bm{w}) = -\bm{X}_i^T(\bm{w}\odot\bm{z}_{i})$ and $\bm{z}_{i}=\bigodot\limits_{j\neq i}^m(\bm{X}_j\bm{u}_j)$.
\STATE Compute $\bm{w}^{t+1}$ by solving the following problem:
\begin{equation}\label{equ-25}
	\begin{aligned}
		& \underset{\bm{w}}{\text{minimize}} && \|\bm{w} - \overline{\bm{w}}\|_2^2  \\
		&  && \|\bm{w}\|_0\leq k_w, \|\bm{w}\|_\infty \leq 1, w_j\geq 0~\forall j
	\end{aligned}
\end{equation}
where $\overline{\bm{w}} = \bm{w}^{t} - \frac{1}{L_w} \nabla_{\bm{w}} f(\bm{u}^{t+1},\bm{w}^{t})$ and the partial derivatives $\nabla_{\bm{w}} f(\bm{u},\bm{w}) = -\bigodot\limits_{i=1}^m(\bm{X}_i\bm{u}_i)$.
\STATE $t = t + 1$
\UNTIL convergence (e.g., $\frac{|\mbox{obj}(t)-\mbox{obj}(t-1)|}{\mbox{obj}(t-1)} < 10^{-5}$).
\RETURN $\bm{u}_i:=\bm{u}_i^{t}$ for all $i$ and $\bm{w}:=\bm{w}^{t}$.
\end{algorithmic}
\end{algorithm}

\textbf{Initialization}.
The initial point setting of Algorithm \ref{Alg-3} is the same as that of Algorithm \ref{Alg-2}, and the description is omitted here.

\textbf{Complexity Analysis.}
The complexity of $\nabla_{\bm{u}_i} f(\bm{u},\bm{w}) = -\bm{X}_i^T(\bm{w}\odot\bm{z}_{i})$ and $\bm{z}_{i}=\bigodot\limits_{j\neq i}^m(\bm{X}_j\bm{u}_j)$ is $\mathcal{O}(\sum_i np_i)$.
The complexity of $\nabla_{\bm{w}} f(\bm{u},\bm{w}) = -\bigodot\limits_{i=1}^m(\bm{X}_i\bm{u}_i)$ is also $\mathcal{O}(\sum_i np_i)$.
So, the complexity of Algorithm \ref{Alg-3} is $\mathcal{O}(tmnp_1+tmnp_2+\cdots+tmnp_m)$, where $t$ is the number of iterations and it is empirically set to 20 in our experiments.

\textbf{Convergence Analysis.}
Similar to Lemma \ref{lem1}, we can also prove that the objective function of the mwsPLS-scheme2 model satisfies the K\L~property.
So, Algorithm \ref{Alg-3} also has theoretical convergence guarantees and globally converges to a critical point.

\section{Experiments}
In this section, we evaluate the effectiveness of these proposed methods for co-module identification and compare them with other competing methods on the synthetic and biomedical data.

\subsection{Competing Methods}\label{chapter-3.1}
We compare the performance of the proposed methods with PLS and its variants:
\begin{itemize}
  \item PLS: It is a classical model without any sparse constrain \cite{rosipal2005overview}:
  \begin{equation*}
  \begin{aligned}
   & \underset{\bm{u},\bm{v}}{\text{maximize}} && \bm{u}^T\bm{X}^T\bm{Y}\bm{v} \\
   & \text{subject to}  && \|\bm{u}\| = \|\bm{v}\| = 1,
  \end{aligned}
  \end{equation*}
  which is solved using the single value decomposition (SVD) method for $\bm{X}^T\bm{Y}$.

  \item PMD-sPLS: It is a sparse model which uses the $\ell_1$-norm constraint, which is inspired by \cite{witten2009extensions}:
  \begin{equation*}
  \begin{aligned}
   & \underset{\bm{u},\bm{v}}{\text{maximize}} && \bm{u}^T\bm{X}^T\bm{Y}\bm{v} \\
   & \text{subject to}  && \|\bm{u}\| = \|\bm{v}\| = 1, \|\bm{u}\| < c_1, \|\bm{v}\| < c_2,
  \end{aligned}
  \end{equation*}
  which is solved by the penalized matrix decomposition (PMD) method \cite{witten2009extensions}.

  \item $\ell_0$-sPLS: It is a sparse PLS model which uses the $\ell_0$-norm constraint, which is inspired by \cite{asteris2016simple,min2016two}:
  \begin{equation*}
  \begin{aligned}
   & \underset{\bm{u},\bm{v}}{\text{maximize}} && \bm{u}^T\bm{X}^T\bm{Y}\bm{v} \\
   & \text{subject to}  && \|\bm{u}\| = \|\bm{v}\| = 1, \|\bm{u}\|_0\leq k_u, \|\bm{v}\|_0\leq k_v,
  \end{aligned}
  \end{equation*}
  which is solved using an alternating iterative algorithm.

  \item $\ell_2/\ell_0$-wsPLS: It is a weighted sparse PLS model which uses the $\ell_2/\ell_0$-norm constrains for the weighed vector, which is inspired by \cite{wenwen2018sparse}.
  \begin{equation*}
  \begin{aligned}
  & \underset{\bm{u},\bm{v},\bm{w}}{\text{maximize}} && \bm{u}^T\bm{X}^T \mbox{diag}(\bm{w}) \bm{Y}\bm{v} \\
  & \text{subject to}  && \|\bm{u}\|_0\leq k_u, \|\bm{v}\|_0\leq k_v, \|\bm{u}\| = \|\bm{v}\| = 1,\\
  &  && \|\bm{w}\|_2 = 1, \|\bm{w}\|_0\leq k_w,
  \end{aligned}
  \end{equation*}
  which is also solved using the proposed Algorithm \ref{Alg-1} with a small change.
\end{itemize}

\begin{figure*}[htbp]
  \centering
  \includegraphics[width=1\linewidth]{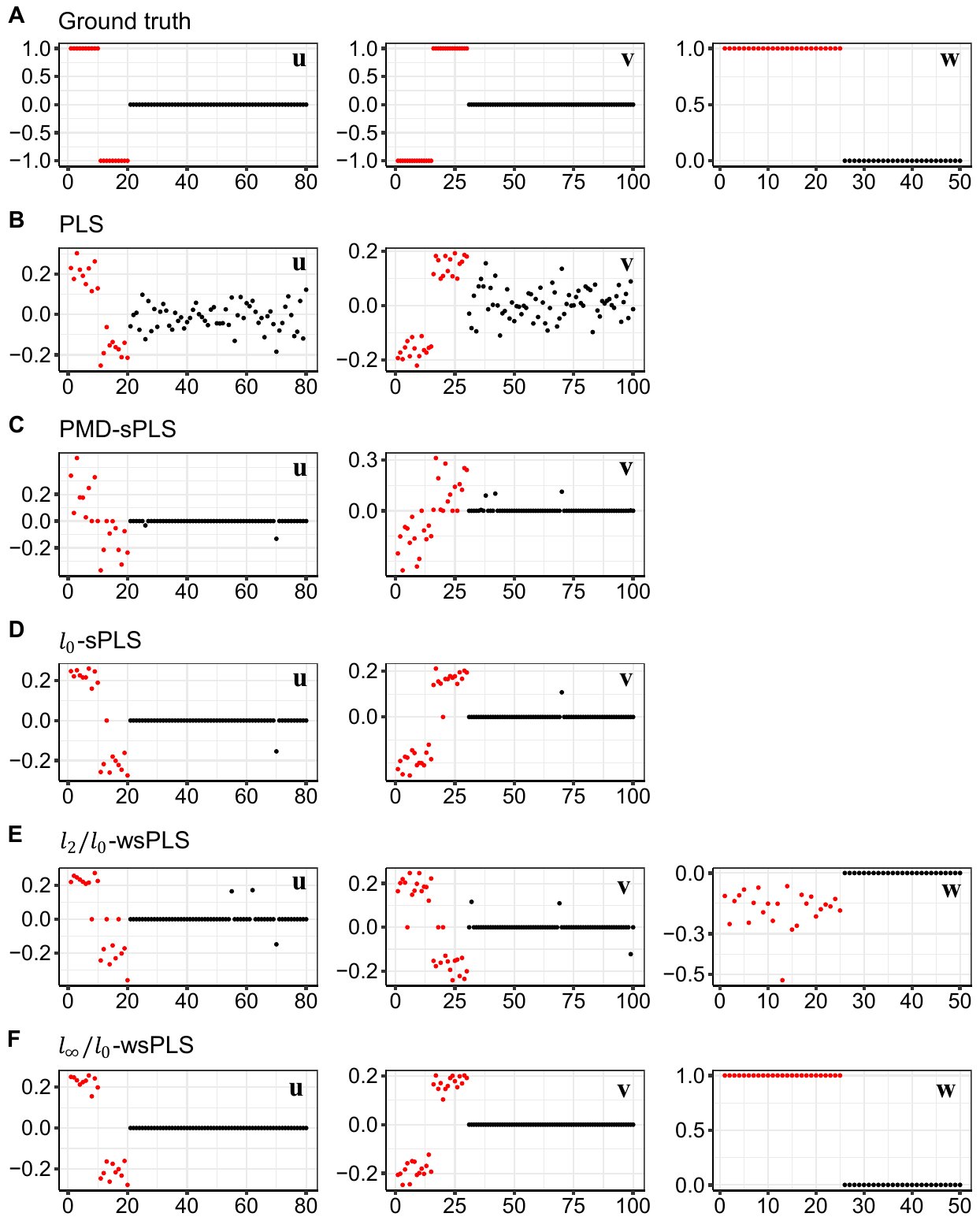}
  \caption{We compare the performance of wsPLS with PLS and its variants on the simulation data I.
  The points (in red) are truly nonzero and the points (in black) are truly zero.
  (A) Showing the true $\bm{u}$, $\bm{v}$ and $\bm{w}$;
  (B) Showing the estimated $\bm{u}$, $\bm{v}$ and $\bm{w}$ by PLS;
  (C) Showing the estimated $\bm{u}$ and $\bm{v}$ by PMD-sPLS;
  (D) Showing the estimated $\bm{u}$ and $\bm{v}$ by $\ell_0$-sPLS;
  (E) Showing the estimated $\bm{u}$, $\bm{v}$ and $\bm{w}$ by $\ell_2/\ell_0$-wsPLS;
  (F) Showing the estimated $\bm{u}$, $\bm{v}$ and $\bm{w}$ by $\ell_\infty/\ell_0$-wsPLS.
  }\label{fig-2}
\end{figure*}

\subsection{Application to Synthetic Data}
To generate the synthetic data of $\bm{X}\in \mathbb{R}^{n\times p}$ and $\bm{Y}\in \mathbb{R}^{n\times q}$ ($n$ is the number of samples, $p$ and $q$ are the number of two different features, respectively), we firstly generated $\bm{u}$, $\bm{v}$ and $\bm{w}$, and then generate them by
\begin{equation}\label{equ-26}
  \begin{array}{rl}
    \bm{X}~~= & \bm{w}\bm{u}^T + \gamma_1 \bm{\epsilon}_1, \\
    \bm{Y}~~= & \bm{w}\bm{v}^T + \gamma_2 \bm{\epsilon}_2,
  \end{array}
\end{equation}
where the elements of $\bm{\epsilon}_1$ and $\bm{\epsilon}_2$ are randomly sampled from a standard normal distribution and $\gamma_1$, $\gamma_2$ are two nonnegative parameters to control the signal-to-noise ratio (SNR) which is defined by
\begin{equation}\label{equ-27}
  \begin{array}{rl}
    \mbox{SNR}_1~= & \frac{\|\bm{w}\bm{u}\|_F^2}{\mathbf{E}(\|\gamma_1 \bm{\epsilon}_1\|_F^2)} = \frac{\|\bm{w}\bm{u}\|_F^2}{\gamma_1^2np}, \\
    \mbox{SNR}_2~= & \frac{\|\bm{w}\bm{v}\|_F^2}{\mathbf{E}(\|\gamma_2 \bm{\epsilon}_2\|_F^2)} = \frac{\|\bm{w}\bm{v}\|_F^2}{\gamma_2^2nq},
  \end{array}
\end{equation}
where $\mathbf{E}(\|\gamma_1 \bm{\epsilon}_1\|_F^2)$ and $\mathbf{E}(\|\gamma_2 \bm{\epsilon}_2\|_F^2)$ denote the expected sum of squares of noise.

\textbf{(1) Simulation data I ($n$ = 50, $p$ = 80, $p$ = 100).}
We firstly generate three vectors $\bm{u}$, $\bm{v}$ and $\bm{w}$ using
\begin{equation*}
  \begin{array}{rl}
    \bm{u}~~= & [rep(1,10),  rep(-1,10), rep(0,p-20)]^{\rm{T}},\\
    \bm{v}~~= & [rep(-1,15), rep(1,15),  rep(0,q-30)]^{\rm{T}},\\
    \bm{w}~~= & [rep(1,25),  rep(0,n-25)]^{\rm{T}},\\
  \end{array}
\end{equation*}
where $rep(x,n)$ denotes a row vector of size $n$ with all elements are equal to $x$.
We then generate the simulation data I using the formula (\ref{equ-26}) with $\mbox{SNR}_1=\mbox{SNR}_2=0.1$.

\textbf{(2) Simulation data II ($n$ = 100, $p$ = 800, $p$ = 1000).}
We firstly generate three vectors $\bm{u}$, $\bm{v}$ and $\bm{w}$ using
\begin{equation*}
  \begin{array}{rl}
    \bm{u}~~= & [rep(1,100),  rep(-1,100), rep(0,p-200)]^{\rm{T}},\\
    \bm{v}~~= & [rep(-1,150), rep(1,150),  rep(0,q-300)]^{\rm{T}},\\
    \bm{w}~~= & [rep(1,50),  rep(0,n-50)]^{\rm{T}}.\\
  \end{array}
\end{equation*}
We then generate the simulation data II using the formula (\ref{equ-26}) with $\mbox{SNR}_1=\mbox{SNR}_2=0.1$.

\textbf{(3) Simulation data III ($n$ = 1000, $p$ = 8000, $p$ = 10000).}
We firstly generate three vectors $\bm{u}$, $\bm{v}$ and $\bm{w}$ using
\begin{equation*}
  \begin{array}{rl}
    \bm{u}~~= & [rep(1,1000),  rep(-1,1000), rep(0,p-2000)]^{\rm{T}},\\
    \bm{v}~~= & [rep(-1,1500), rep(1,1500),  rep(0,q-3000)]^{\rm{T}},\\
    \bm{w}~~= & [rep(1,250),   rep(0,n-250)]^{\rm{T}}.\\
  \end{array}
\end{equation*}
We then generate the simulation data III using the formula (\ref{equ-26}) with $\mbox{SNR}_1=\mbox{SNR}_2=0.1$.

We compare the performance of $\ell_\infty/\ell_0$-wsPLS with $\ell_2/\ell_0$-wsPLS \cite{wenwen2018sparse}, two sparse PLS methods including $\ell_0$-sPLS \cite{asteris2016simple} and PMD-sPLS with $\ell_1$-penalty \cite{witten2009penalized}, and the classical PLS \cite{rosipal2005overview}.
All the experiments are performed on a laptop (2.8-GHz Intel 8-Core CPU and 16-GB RAM).
Since both PLS and sPLS cannot estimate $\bm{w}$, we assume that $\bm{w} = \bm{1}$ for the convenience of comparison.

For the simulation data I, we set parameters $k_u = 20$, $k_v = 30$ and $k_w = 25$ for $\ell_\infty$-wsPLS and $\ell_2/\ell_0$-wsPLS.
For comparison, we set $k_u = 20$, $k_v = 30$ for $\ell_0$-sPLS and ensure that $\bm{u}$ and $\bm{v}$ of PMD-sPLS output have the same sparsity level.
For the simulation data II, we set parameters $k_u = 200$, $k_v = 300$ and $k_w = 50$ for $\ell_\infty$-wsPLS and $\ell_2/\ell_0$-wsPLS.
Meanwhile, we set $k_u = 200$, $k_v = 300$ for $\ell_0$-sPLS and ensure that $\bm{u}$ and $\bm{v}$ of PMD-sPLS output have the same sparsity level.
For the simulation data III, we set parameters $k_u = 2000$, $k_v = 3000$ and $k_w = 250$ for $\ell_\infty$-wsPLS and $\ell_2/\ell_0$-wsPLS.
For comparison, we set $k_u = 2000$, $k_v = 3000$ for $\ell_0$-sPLS and ensure that $\bm{u}$ and $\bm{v}$ of PMD-sPLS output have the same sparsity level.

A optimal method should identify the true non-zero patterns of $\bm{u}$, $\bm{v}$ and $\bm{w}$ as Figure \ref{fig-2}~(A).
We evaluate the performance of all methods by three evaluation metrics including true positive rate (TPR), true negative rate (TNR) and accuracy (ACC):
\begin{equation}\label{equ-28}
  \mbox{TPR} = \frac{\mbox{TP}}{\mbox{P}},~\mbox{TNR} = \frac{\mbox{TN}}{\mbox{N}},~\mbox{ACC} = \frac{\mbox{TP+TN}}{\mbox{P+N}},
\end{equation}
where \mbox{P} denotes the number of positive samples (i.e., non-zero values), \mbox{N} denotes the number of negative samples (i.e., zero values), \mbox{TP} denotes the number of true positive, \mbox{TN} denotes the number of true negative.
\begin{table*}[htbp]
\caption{Performance comparison of different methods in terms of ACC, TPR, TNR and running time (in seconds) on three simulated datasets. Mean and standard derivation of each performance measure is calculated based on 20 simulation runs. The evaluation metrics (ACC, TPR, and TNR) are defined in Eq.(\ref{equ-28}). For example $ACC\_u$, $ACC\_v$ and $ACC\_w$ represent the evaluation of $\bm{u}$, $\bm{v}$ and $\bm{w}$ variables respectively, while $ACC\_all$  represents the evaluation of all variables (i.e. the combination of $\bm{u}$, $\bm{v}$, $\bm{w}$).}\label{tab-2}
\centering
\resizebox{2\columnwidth}{!}{
\begin{tabular}{@{\vrule height8pt depth2pt width0pt} l|cccc|cccc|cccc|c}
  \toprule
  Data I    & \multicolumn{13}{c}{Simulation data I: $n$ = 50, $p$ = 80, $q$ = 100}  \\
  \hline
  Metrics & \multicolumn{4}{c|}{ACC} & \multicolumn{4}{c|}{TPR} & \multicolumn{4}{c|}{TNR} & \multicolumn{1}{c}{Time} \\
  \hline
  & ACC\_all &ACC\_u & ACC\_v & ACC\_w & TPR\_all & TPR\_u & TPR\_v & TPR\_w & TNR\_all & TNR\_u & TNR\_v & TNR\_w &Time\\
  \hline
  PLS & 0.326 & 0.250 & 0.300 & 0.500 & \bf{1.000} & \bf{1.000} & \bf{1.000} & \bf{1.000} & 0.000 & 0.000 & 0.000 & 0.000 & 0.004 \\
      &(0.000) &(0.000) &(0.000) &(0.000) &(0.000) &(0.000) &(0.000) &(0.000) &(0.000) &(0.000) &(0.000) &(0.000) &(0.010)\\
  PMD-sPLS & 0.800 & 0.916 & 0.856 & 0.500 & 0.880 & 0.815 & 0.823 & 1.000 & 0.761 & 0.950 & 0.870 & 0.000 & 0.016 \\
      &(0.016) &(0.023) &(0.034) &(0.000) &(0.025) &(0.067) &(0.040) &(0.000) &(0.016) &(0.015) &(0.039) &(0.000) &(0.007) \\
  $\ell_0$-sPLS & 0.870 & 0.980 & 0.966 & 0.500 & 0.967 & 0.960 & 0.943 & 1.000 & 0.823 & 0.987 & 0.976 & 0.000 & 0.029 \\
      &(0.010) &(0.022) &(0.016) &(0.000) &(0.015) &(0.044) &(0.026) &(0.000) &(0.007) &(0.015) &(0.011) &(0.000) &(0.008) \\
  $\ell_2/\ell_0$-wsPLS & 0.710 & 0.728 & 0.692 & 0.716 & 0.555 & 0.455 & 0.487 & 0.716 & 0.785 & 0.818 & 0.780 & 0.716 & 0.023 \\
      &(0.135) &(0.125) &(0.118) &(0.214) &(0.207) &(0.249) &(0.196) &(0.214) &(0.100) &(0.083) &(0.084) &(0.214) &(0.007) \\
  $\ell_\infty/\ell_0$-wsPLS & \bf{0.979} & \bf{0.980} & \bf{0.972} & \bf{0.992} & 0.968 & 0.960 & 0.953 & 0.992 & \bf{0.985} & \bf{0.987} & \bf{0.980} & \bf{0.992} & 0.056 \\
      &(0.016) &(0.019) &(0.016) &(0.024) &(0.024) &(0.037) &(0.027) &(0.024) &(0.012) &(0.012) &(0.011) &(0.024) &(0.007) \\
  \hline
  \hline

  Data II   & \multicolumn{13}{c}{Simulation data II: $n$ = 100, $p$ = 800, $q$ = 1000}  \\
  \hline
    & ACC\_all &ACC\_u & ACC\_v & ACC\_w & TPR\_all & TPR\_u & TPR\_v & TPR\_w & TNR\_all & TNR\_u & TNR\_v & TNR\_w &Time\\
  \hline
  PLS & 0.289 & 0.250 & 0.300 & 0.500 & \bf{1.000} & \bf{1.000} & \bf{1.000} & \bf{1.000} & 0.000 & 0.000 & 0.000 & 0.000 & 0.065 \\
  &(0.000) &(0.000) &(0.000) &(0.000) &(0.000) &(0.000) &(0.000) &(0.000) &(0.000) &(0.000) &(0.000) &(0.000) &(0.008) \\
  PMD-sPLS & 0.868 & 0.795 & 0.963 & 0.500 & 0.816 & 0.625 & 0.913 & 1.000 & 0.889 & 0.853 & 0.984 & 0.000 & 0.079 \\
  &(0.006) &(0.010) &(0.006) &(0.000) &(0.008) &(0.014) &(0.013) &(0.000) &(0.006) &(0.010) &(0.004) &(0.000) &(0.012) \\
  $\ell_0$-sPLS & 0.926 & 0.890 & 0.997 & 0.500 & 0.917 & 0.780 & 0.995 & 1.000 & 0.929 & 0.927 & 0.998 & 0.000 & 0.114 \\
     &(0.005) &(0.013) &(0.001) &(0.000) &(0.009) &(0.026) &(0.002) &(0.000) &(0.004) &(0.009) &(0.001) &(0.000) &(0.009)  \\
  $\ell_2/\ell_0$-wsPLS & 0.617 & 0.627 & 0.605 & 0.654 & 0.338 & 0.254 & 0.342 & 0.654 & 0.730 & 0.751 & 0.718 & 0.654 & 0.026 \\
     &(0.027) &(0.013) &(0.035) &(0.179) &(0.047) &(0.025) &(0.059) &(0.179) &(0.019) &(0.008) &(0.025) &(0.179) &(0.008)  \\
  $\ell_\infty/\ell_0$-wsPLS & \bf{0.953} & \bf{0.891} & \bf{0.997} & \bf{1.000} & 0.918 & 0.783 & 0.995 & 1.000 & \bf{0.967} & \bf{0.927} & \bf{0.998} & \bf{1.000} & 0.117 \\
     &(0.005) &(0.013) &(0.001) &(0.000) &(0.009) &(0.026) &(0.002) &(0.000) &(0.004) &(0.009) &(0.001) &(0.000) &(0.010)  \\
  \hline
  \hline

  Data III   & \multicolumn{13}{c}{Simulation data III: $n$ = 500, $p$ = 8000, $q$ = 10000}  \\
  \hline
   & ACC\_all &ACC\_u & ACC\_v & ACC\_w & TPR\_all & TPR\_u & TPR\_v & TPR\_w & TNR\_all & TNR\_u & TNR\_v & TNR\_w &Time\\
  \hline
  PLS & 0.284 & 0.250 & 0.300 & 0.500 & \bf{1.000} & \bf{1.000} & \bf{1.000} & \bf{1.000} & 0.000 & 0.000 & 0.000 & 0.000 & 11.965 \\
   &(0.000) &(0.000) &(0.000) &(0.000) &(0.000) &(0.000) &(0.000) &(0.000) &(0.000) &(0.000) &(0.000) &(0.000) &(0.097)\\
  PMD-sPLS & 0.940 & 0.919 & 0.979 & 0.500 & 0.892 & 0.821 & 0.931 & 1.000 & 0.960 & 0.952 & 1.000 & 0.000 & 11.999 \\
   &(0.002) &(0.004) &(0.001) &(0.000) &(0.003) &(0.007) &(0.003) &(0.000) &(0.001) &(0.003) &(0.000) &(0.000) &(0.205)  \\
  $\ell_0$-sPLS & 0.976 & 0.977 & 1.000 & 0.500 & 0.982 & 0.954 & 1.000 & 1.000 & 0.974 & 0.985 & 1.000 & 0.000 & 26.078 \\
   &(0.001) &(0.002) &(0.000) &(0.000) &(0.001) &(0.004) &(0.000) &(0.000) &(0.001) &(0.001) &(0.000) &(0.000) &(0.238)  \\
  $\ell_2/\ell_0$-wsPLS & 0.921 & 0.907 & 0.928 & 0.991 & 0.860 & 0.813 & 0.880 & 0.991 & 0.945 & 0.938 & 0.949 & 0.991 & 0.360 \\
   &(0.138) &(0.138) &(0.143) &(0.019) &(0.243) &(0.277) &(0.239) &(0.019) &(0.096) &(0.092) &(0.102) &(0.019) &(0.029)  \\
  $\ell_\infty/\ell_0$-wsPLS & \bf{0.990} & \bf{0.977} & \bf{1.000} & \bf{1.000} & 0.982 & 0.954 & 1.000 & 1.000 & \bf{0.993} & \bf{0.985} & \bf{1.000} & \bf{1.000} & 5.948 \\
   &(0.001) &(0.002) &(0.000) &(0.000) &(0.001) &(0.004) &(0.000) &(0.000) &(0.001) &(0.001) &(0.000) &(0.000) &(0.057) \\
  \toprule
\end{tabular}\label{tab:5}
}
\end{table*}

Figure \ref{fig-2} shows the estimated patterns of $\bm{u}$, $\bm{v}$ and $\bm{w}$ on the simulation data I.
Table \ref{tab-2} shows the detailed performance results of different methods in terms of ACC, TPR, TNR and running time (seconds).
We find
(1) the sparse PLS methods including PMD-sPLS and $\ell_0$-sPLS cannot detect the pattern of $\bm{w}$ (samples) and they identify $\bm{u}$ and $\bm{v}$ with some wrong points;
(2) compared to $\ell_2/\ell_0$-wsPLS, $\ell_\infty$-wsPLS can identify more correct pattern of $\bm{w}$.
In short, $\ell_\infty$-wsPLS not only discovers the true non-zero patterns for $\bm{u}$, $\bm{v}$, but also identifies the true non-zero patterns for $\bm{w}$ (samples).

\subsection{Application to Paired DNA Copy Number Variation and Gene Expression Data}
We demonstrate the $\ell_\infty/\ell_0$-wsPLS method on the comparative genomic hybridization (CGH) data with matched samples from ``PMA'' R package \cite{witten2009penalized}, which contains the matched DNA copy number variation (CNV) data $\bm{X}\in \mathbb{R}^{89 \times 2149}$ and gene expression data $\bm{Y}\in \mathbb{R}^{89 \times 19672}$. Here, our aim is to identify a gene set whose expression is strongly correlated with CNV across a subset of samples.

We apply $\ell_\infty/\ell_0$-wsPLS on the CGH dataset with $k_u = 50$, $k_v = 150$ and $k_w = 89*0.8$ and identify a co-expression module of CNV and gene which contains 50 CNV regions, 150 genes and 71 patients.
For comparison, we also apply $\ell_0$-sPLS with $k_u = 50$ and $k_v = 150$, PMD-sPLS with $c_1=0.128$, $c_2=0.0665$ for learning $\bm{u}$ and $\bm{v}$ with the same sparse level. The setting of $\ell_2/\ell_0$-wsPLS is same as $\ell_\infty/\ell_0$-wsPLS.
We can first see that the co-module identified by $\ell_\infty/\ell_0$-wsPLS has a strong co-correlation pattern (Figure \ref{fig-5} (A), (B) and (C)).
Second, we find that the proposed $\ell_\infty/\ell_0$-wsPLS obtains the largest Pearson correlation coefficient (PCC) $r=0.901$, while
$\ell_2/\ell_0$-wsPLS obtains $r=0.816$, $\ell_0$-sPLS obtains $r=0.802$ and PMD-sPLS obtains $r=0.779$ (Figure \ref{fig-5} (D)).
These results show that wsPLS is more effective to capture the latent patterns by selecting a specific subset of samples than $\ell_0$-sPLS and PMD-sPLS methods.

\begin{figure*}[htbp]
	\centering
	\includegraphics[width=1\linewidth]{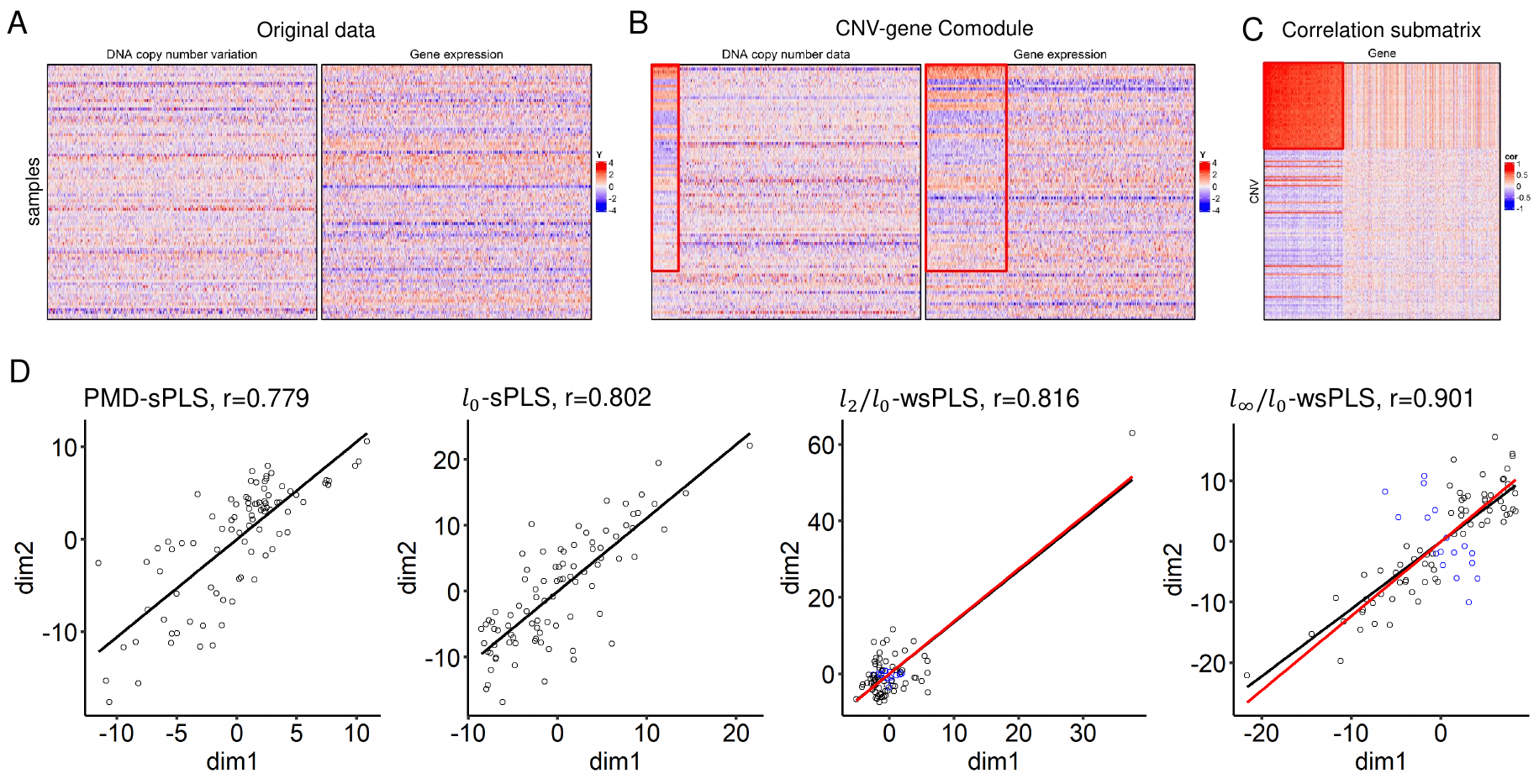}
	\caption{Results on the paired DNA copy number variation and gene expression data of CGH.
		(A) Showing heatmaps of input data.
		(B) Showing heatmaps of the identified co-module (circled in red lines) by wsPLS and randomly selected features for comparison.
		(C) Showing heatmaps of correlation submatrix for the identified co-module.
		(D) Showing scatterplots of dim1 ($\bm{Xu}$) vs. dim2 ($\bm{Xv}$) for four different methods,
		including PMD-sPLS and $\ell_0$-sPLS, $\ell_2/\ell_0$-wsPLS and $\ell_\infty/\ell_0$-wsPLS (see section \ref{chapter-3.1} for more details) where
		the blue points in the third and fourth pictures correspond to $w_j=0$ and
		the black line is fitted using all data points and the red line is fitted using only black data points, and Pearson correlation coefficient (PCC) $r$ is displayed.
	}\label{fig-5}
\end{figure*}
\begin{figure*}[htbp]
	\centering
	\includegraphics[width=1\linewidth]{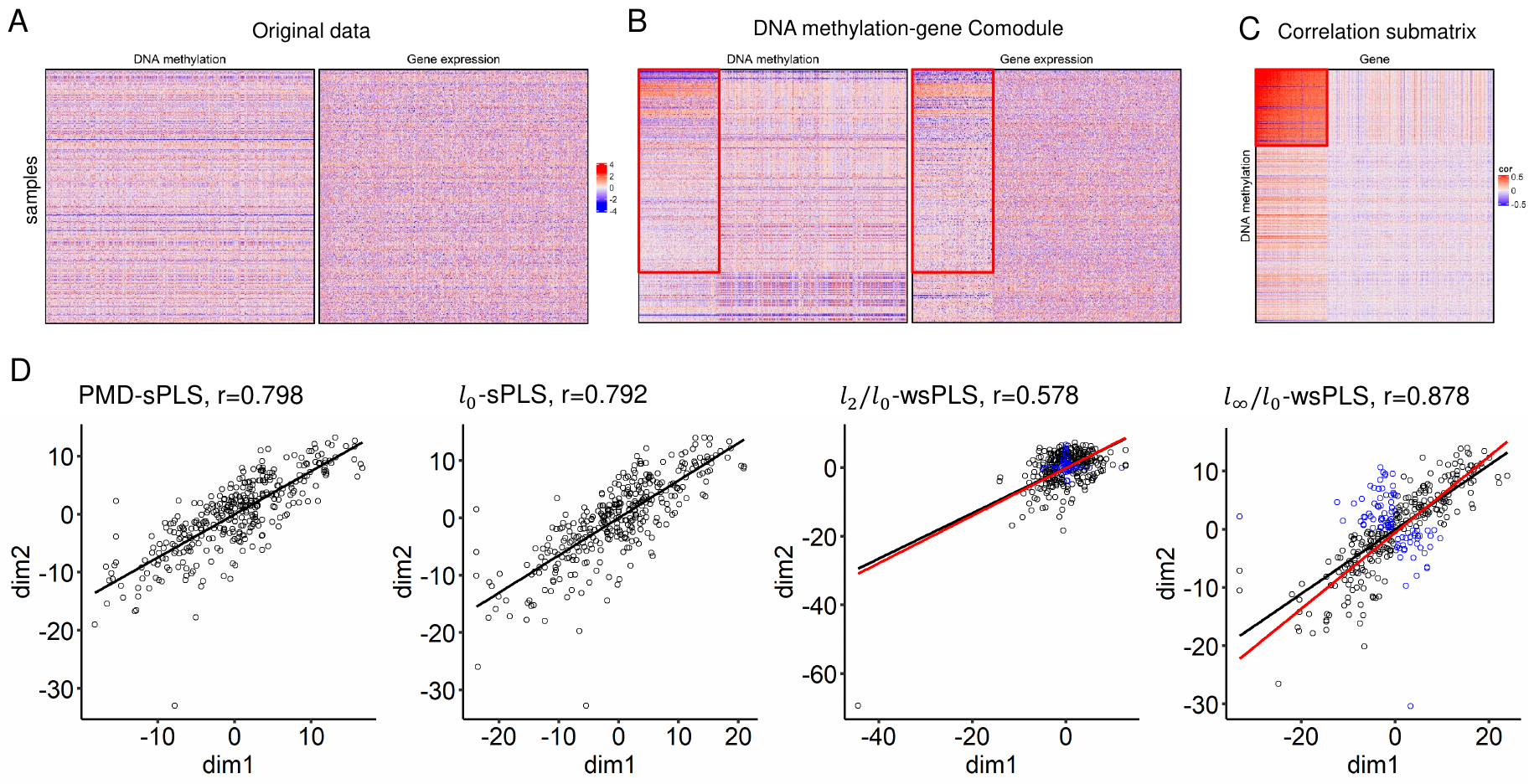}
	\caption{Results on the paired DNA methylation and gene expression data of LUAD.
		(A) Showing heatmaps of input data.
		(B) Showing heatmaps of the identified co-module (circled in red lines) by wsPLS and randomly selected features for comparison.
		(C) Showing heatmaps of correlation submatrix for the identified co-module.
		(D) Showing scatterplots of dim1 ($\bm{Xu}$) vs. dim2 ($\bm{Xv}$) for four different methods, where the blue points in the third and fourth pictures correspond to $w_j=0$ and the black line is fitted using all data points and the red line is fitted using only black data points, and PCC $r$ is displayed.
	}\label{fig-6}
\end{figure*}

\begin{figure*}[htbp]
	\centering
	\includegraphics[width=0.77\linewidth]{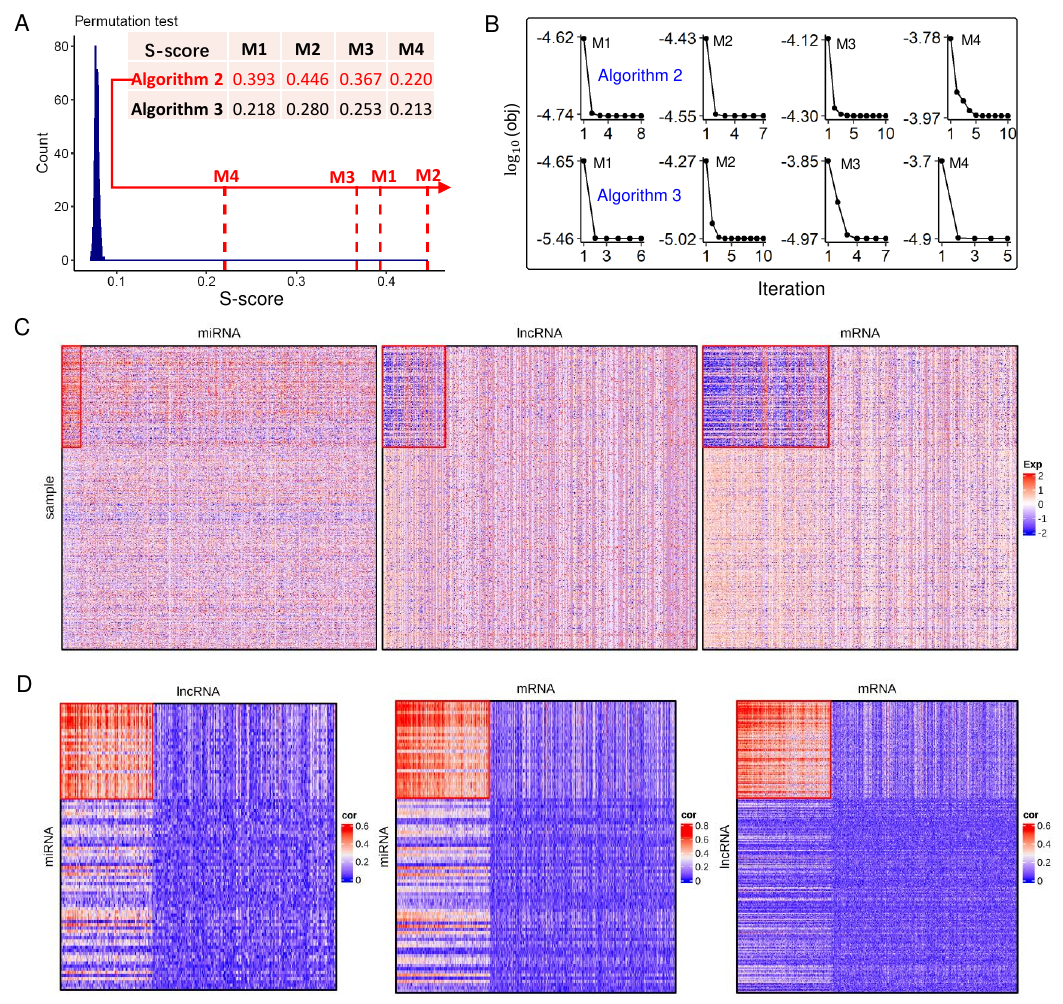}
	\caption{Application of the proposed Algorithm \ref{Alg-2} and \ref{Alg-3} on the BRCA miRNA-lncRNA-mRNA dataset.
    (A) Distribution of the $S$ scores (See Eq. (\ref{equ-30}) for definition). The $S$ scores of identified modules are significantly greater than those of random ones (Permutation test $P$ < 0.01). The table in the upper right corner corresponds to the $S$ scores of the co-modules identified by mwsPLS-scheme1 (Algorithm \ref{Alg-2}) and mwsPLS-scheme2 (Algorithm \ref{Alg-3}).
    (B) Showing the convergence curves of the proposed Algorithm \ref{Alg-2} and \ref{Alg-3}.
	(C) Showing heatmaps of the first co-module identified by Algorithm \ref{Alg-2} (see red rectangular areas) and randomly selected features for comparison.
	(D) Showing heatmaps of the absolute values of PCCs for these correlation matrices on (miRNA, lncRNA), (miRNA, mRNA) and (lncRNA, mRNA), where the three red rectangular areas correspond to the first identified co-module by mwsPLS-scheme1.
	}\label{fig-7}
\end{figure*}

\subsection{Application to Paired DNA Methylation and Gene Expression Data}
Aberrant DNA methylation plays an important role in the development of cancer, which has been regarded as one of the biomarkers of cancer \cite{saghafinia2018pan}.
In recent years, it has become very important to detect the cross-talk mechanism between DNA methylation and gene expression \cite{wang2014network,li2014potential}.
To this end, we obtain the paired DNA methylation and gene expression dataset of Lung Adenocarcinoma (LUAD) in The Cancer Genome Atlas (TCGA) \cite{cancer2014comprehensive,li2014potential}. To remove some noise methylation probes and genes, we extract the top 5000 methylation probes and the top 5000 genes with largest  variance for further analysis.
Finally, the LUAD dataset contains the DNA methylation $\bm{X}\in \mathbb{R}^{350 \times 5000}$ and gene expression data $\bm{Y}\in \mathbb{R}^{350 \times 5000}$.
We apply $\ell_\infty/\ell_0$-wsPLS methods on the LUAD dataset with $k_u = 150$, $k_v = 150$ and $k_w = 350*0.8$ to identify methylation-gene co-module.
For comparison, we also apply $\ell_0$-sPLS with $k_u = 150$ and $k_v = 150$, PMD-sPLS with $c_1=0.1327$, $c_2=0.1375$ on the LUAD data for learning $\bm{u}$ and $\bm{v}$ with the same sparse level.
The setting of $\ell_2/\ell_0$-wsPLS is same as $\ell_\infty/\ell_0$-wsPLS.
The co-module identified by $\ell_\infty/\ell_0$-wsPLS has a strong co-correlation pattern (Figure \ref{fig-6} (A), (B) and (C)).
In addition, we find that $\ell_\infty/\ell_0$-wsPLS obtains the largest PCC $r=0.878$, while
$\ell_2/\ell_0$-wsPLS obtains $r=0.578$, $\ell_0$-sPLS obtains $r=0.792$ and PMD-sPLS obtains $r=0.798$ (Figure \ref{fig-6} (D)).
These results show that wsPLS is more effective to capture the latent patterns of canonical vectors by selecting a specific subset of samples than $\ell_0$-sPLS and PMD-sPLS methods.

\subsection{Application to Paired miRNA-lncRNA-mRNA Expression Data}
To evaluate the mwsPLS method on triple datasets with matched samples, we obtain the miRNA-lncRNA-mRNA transcript expression datasets across 751 Breast invasive carcinoma (BRCA) patients from TCGA database \cite{cancer2012comprehensive}.
We remove some noise features (miRNAs, lncRNAs and mRNAs) by using standard deviation threshold. Finally, we obtain three data matrices including the miRNA expression matrix $\bm{X}_1\in \mathbb{R}^{751 \times 581}$, the lncRNA expression matrix $\bm{X}_2\in \mathbb{R}^{751 \times 3782}$ and the mRNA expression matrix $\bm{X}_2\in \mathbb{R}^{751 \times 6200}$.

We apply mwsPLS-scheme1 (Algorithm \ref{Alg-2}) and mwsPLS-scheme2 (Algorithm \ref{Alg-3}) on the BRCA miRNA-lncRNA-mRNA dataset with $k_1 = 30$ (miRNAs), $k_2 = 100$ (lncRNAs), $k_3 = 200$ (mRNAs) and $k_w=190$ (about $1/4$ of all samples) to identify four subtype-specific miRNA-lncRNA-mRNA co-modules.
We define an average correlation score $S$ to quantify the co-correlation level for a given co-module $M(I,J,K,T)$  where $I$, $J$ and $K$ represent three different feature sets (i.e., miRNA subset, lncRNA subset and mRNA subset) and $T$ represents sample subset:
\begin{equation}\label{equ-30}
	S = \frac{\sum\limits_{i\in I, j\in J}|\text{cor}_{(i,j)}|+\sum\limits_{i\in I, k\in K}|\text{cor}_{(i,k)}|+\sum\limits_{j\in J, k\in K}|\text{cor}_{(j,k)}|}{N},
\end{equation}
where $N=|I|\cdot|J|+|I|\cdot|K|+|J|\cdot|K|$ is the number of all the possible pairs for the three types in the co-module $M(I,J,K,T)$ and $cor_{(i,j)}$ is used to calculate the pair-wise Pearson correlation coefficient (PCC) based on their corresponding expression data on the sample subset $T$. Simply, the $S$ is defined as the average value of the absolute PCCs for all miRNA-lncRNA, miRNA-mRNA, and lncRNA-mRNA pairs.
We apply a permutation test to evaluate the significance level of $S$ score for a given co-module. We firstly construct 1000 random co-modules by randomly sampling the miRNA set, lncRNA set and sample set, and then compute the $S$ scores for these random co-modules.
We find that all the $S$ scores of the identified co-modules by Algorithms \ref{Alg-2} and \ref{Alg-3} are significantly higher than the scores of the random co-modules (Figure \ref{fig-7} (A)). Moreover, the objective function curves show the convergence of Algorithm \ref{Alg-2} and \ref{Alg-3} (Figure \ref{fig-7}).
As an example, we show the expressed pattern of the first co-module identified by mwsPLS-scheme1 (Algorithm \ref{Alg-2}) in Figure \ref{fig-7} (C).
Meanwhile, the correlation matrices (heatmaps) of any two types of molecules in the identified co-module are shown in Figure \ref{fig-7} (D).
In terms of the $S$ scores in Figure \ref{fig-7} (A), mwsPLS-scheme1 (Algorithm \ref{Alg-2}) performs as well or better than mwsPLS-scheme2 (Algorithm \ref{Alg-3}).

To demonstrate the biological function of the gene sets from these identified co-modules by mwsPLS-scheme1. The co-modules identified by mwsPLS-scheme1 are used for further biological analysis. We retrieve the GO biological process data information from Molecular Signatures Database (MSigDB, http://www.gsea-msigdb.org/gsea/msigdb/index.jsp).
Each GO biological process is consisted by a set of functionally gene set.
The hypergeometric test is used for the statistical analysis.
The gene sets of the identified co-modules are significantly enriched a series of important biological processes (Table \ref{tab-3}).

\begin{table}[hbpt]
	\caption{Enrichment analysis of GO biological process terms on the gene sets of co-modules which are identified by Algorithm \ref{Alg-2} on BRCA dataset. We present top five terms of each co-module. ``\#gene'' denotes the number of genes.}\label{tab-3}
	\resizebox{\columnwidth}{!}{
		\begin{threeparttable}
			\centering
			\begin{tabular}{cp{4.4cm}cc}
				\hline
				Module & Biological process term & \#gene & P-value \\
				\hline
				M1 & microtubule based process &  13 & 9.43e-06 \\
				& microtubule cytoskeleton organization &  10 & 8.03e-06 \\
				& axonemal dynein complex assembly &   5 & 5.25e-05 \\
				& axoneme assembly &   6 & 1.11e-04 \\
				& regulation of microtubule based process &   6 & 2.17e-04 \\
				M2 & immune system process & 104 & 4.52e-54 \\
				& regulation of immune system process &  89 & 9.71e-53 \\
				& immune response &  79 & 2.53e-42 \\
				& regulation of immune response &  59 & 1.88e-38 \\
				& positive regulation of immune system process &  59 & 6.98e-37 \\
				M3 & regulation of cellular component movement &  24 & 7.68e-08 \\
				& response to ionizing radiation &   6 & 4.10e-06 \\
				& positive regulation of cell proliferation &  25 & 3.54e-06 \\
				& regulation of epithelial cell proliferation &  13 & 5.12e-06 \\
				& positive regulation of locomotion &  15 & 8.19e-06 \\
				M4 & extracellular structure organization &  19 & 1.04e-09 \\
				& organ morphogenesis &  33 & 8.30e-09 \\
				& skeletal system development &  21 & 7.56e-08 \\
				& regulation of multicellular organismal development &  43 & 5.22e-08 \\
				& tissue development &  43 & 7.11e-08 \\
				\hline
			\end{tabular}
		\end{threeparttable}
	}
\end{table}

For example, the gene set in the co-module 2 (M2) is significantly enriched in the immune related biological processes including \textit{immune system process} ($p = 4.52e-54$), \textit{regulation of immune system process} ($p = 9.71e-53$) and \textit{immune response} ($p = 2.53e-42$), which have been reported to be related to breast cancer \cite{desmedt2008biological,bates2018mechanisms}.
The gene set in the co-module 3 (M3) is significantly enriched in the \textit{regulation of epithelial cell proliferation} ($p = 5.12e-06$), which has been reported to be associated with breast cancer \cite{cichon2015regulation}.
These results show that the mwsPLS methods could effectively identify the biologically related miRNA-lncRNA-mRNA comodules.

\section{Conclusion}
In this paper, we propose an $\ell_\infty/\ell_0$-wsPLS model for joint sample and feature selection, which uses the $\ell_\infty/\ell_0$-norm constraints for sample selection.
It is difficult to solve the $\ell_\infty/\ell_0$-wsPLS model because the $\ell_\infty/\ell_0$-norm constraints are non-convex and non-smooth.
Fortunately, we prove that the objective function of the $\ell_\infty/\ell_0$-wsPLS model satisfies the K\L~property.
Based on this, we propose a block proximal gradient algorithm based on the PALM framework to solve it and show its convergence property.
Furthermore, we extend the $\ell_\infty/\ell_0$-wsPLS model for data fusion on three or more data matrices and propose two mwsPLS models by a sum way and a product way.
Similarly, we propose two block proximal gradient algorithms to solve them and prove that they all have theoretical convergence guarantees.
The numerical and biomedical data experiments demonstrate the efficiency of our methods.
The source code of the proposed wsPLS and mwsPLS methods is available from https://github.com/wenwenmin/wsPLS.

\section*{Acknowledgment}
This work was supported by the National Science Foundation of China (62262069, 61802157, 61902372),
the Open Project Program of Yunnan Key Laboratory of Intelligent Systems and Computing (No. ISC22Z03).

\balance
\small{
\bibliographystyle{IEEEtran}
\bibliography{IEEEabrv,Reference}
}
\begin{IEEEbiography}[{\includegraphics[width=1in,height=1.25in,clip,keepaspectratio]{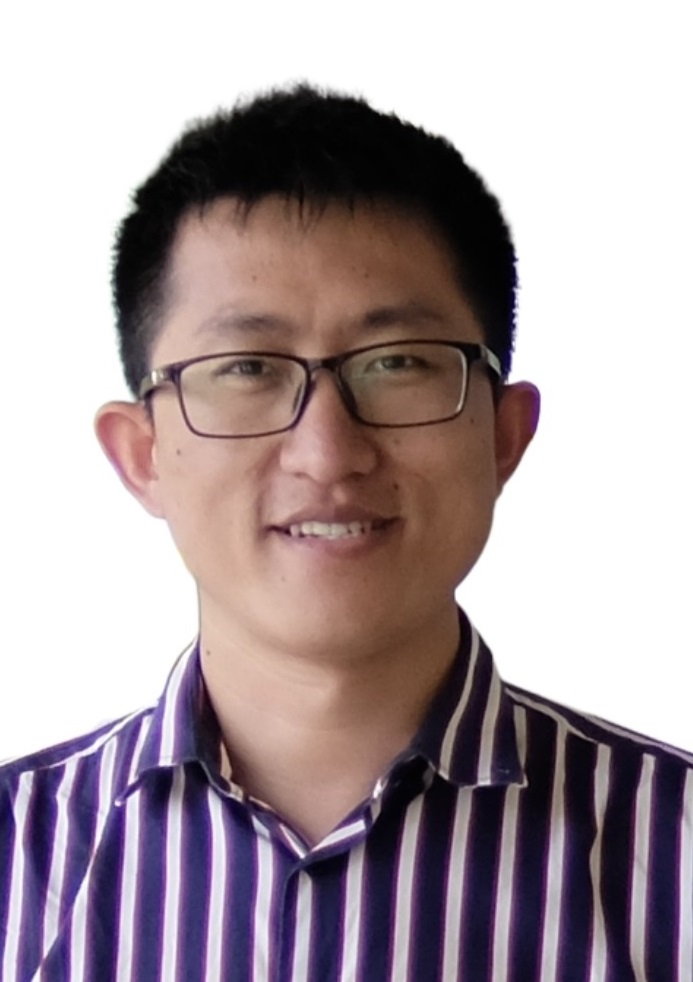}}]{Wenwen Min}
is currently an associate professor in the School of Information Science and Engineering, Yunnan University.
He received the Ph.D. degree in Computer Science from the School of Computer Science, Wuhan University, in December, 2017.
He was a visiting Ph.D student at the Academy of Mathematics and Systems Science, Chinese Academy of Sciences from 2015 to 2017.
He was a Postdoctoral Researcher with the School of Science and Engineering, The Chinese University of Hong Kong, Shenzhen, China, from 2019 to 2021.
His current research interests include data mining, machine learning, sparse optimization and bioinformatics.
He has authored about 20 papers in journals and conferences, such as the
IEEE TKDE,
IEEE TNNLS,
PLoS Comput Biol,
Bioinformatics, and IEEE/ACM TCBB, etc.
\end{IEEEbiography}

\begin{IEEEbiography}[{\includegraphics[width=1in,height=1.25in,clip,keepaspectratio]{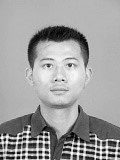}}]{Taosheng Xu} is currently an associate professor in the Hefei Institutes of Physical Science, Chinese Academy of Sciences. He received his BSc (2009) from Hefei University of Technology. He received his PhD degree in Control Science and Engineering in 2016 at University of Science and Technology of China. Dr. Xu is a Postdoctoral Research Fellow at the Arieh WARSHEL Institute for Computational Biology, The Chinese University of Hong Kong, Shenzhen from 2020 to 2022.
His current research interests include data mining, machine learning, bioinformatics and cancer genomics.
\end{IEEEbiography}

\begin{IEEEbiography}[{\includegraphics[width=1in,height=1.25in,clip,keepaspectratio]{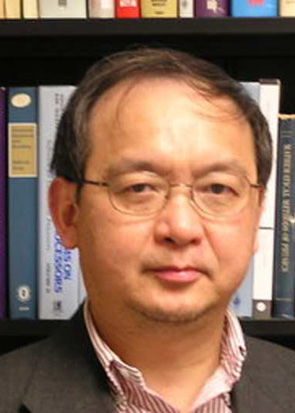}}]{Chris Ding}
received his Ph.D. degree in Columbia University in 1987. He is currently a chair professor at Chinese University of Hong Kong, Shenzhen. Before that, he worked at California Institute of Technology, Jet Propulsion Lab, Lawrence Berkeley Lab, and University of Texas. His areas are machine learning, data mining, computational biology and high performance computing. He proposed L21 norm which is widely used in machine learning. Made significant contributions on principal component analysis, non-negative matrix factorization and feature selection. Designed the minimum redundancy maximum relevance feature selection algorithm that is widely adopted, e.g., by Uber; Two related papers were cited 11,700 times. Published over 200 research papers with over 50,000 citations.
\end{IEEEbiography}
\end{document}